\documentclass[twoside, 11pt]{article}
\usepackage{amsfonts,amsmath,amssymb, amsthm}
\usepackage{hyperref}
\usepackage{graphicx}
\usepackage{color, soul}
 \usepackage{fullpage}
\usepackage{cite}
\usepackage{enumerate}

\usepackage{epstopdf}
\usepackage{algorithm}
\usepackage{algorithmicx}
\usepackage{algpseudocode}

 \newtheorem{theorem}{Theorem}[section]

 \newtheorem{lemma}[theorem]{Lemma}

 \newtheorem{remark}{Remark}
 \newtheorem{defi}[theorem]{Definition}

\newcommand{\be}{\begin{equation}}
\newcommand{\ee}{\end{equation}}

\begin{document}

 \title{{\bf Gradient Estimation with Simultaneous Perturbation and Compressive Sensing} \\ \vspace{.1in} (Gradient Estimation)}

 \vspace{.2in}

\author{Vivek S.\ Borkar, Vikranth R. Dwaracherla, Neeraja Sahasrabudhe \footnote{Department of Electrical Engineering, Indian Institute of Technology Bombay, Powai, Mumbai 400076, India ((borkar.vs, vikranthacad, neeraja.sigma)@gmail.com). Authors listed alphabetically. Work of VSB supported in part by a J.\ C.\ Bose Fellowship and a grant `Approximation of high dimensional optimization and control problems' from Department of Science and Technology, Government of India.}}

 \vspace{.2in}

 \date{}

 \maketitle


\begin{abstract}%
 We propose a scheme for finding a \lq \lq good" estimator for the gradient of a function on a high-dimensional space with few function evaluations. Often such functions are not sensitive in all coordinates and the gradient of the function is almost sparse. We propose a method for gradient estimation that combines ideas from Spall's Simultaneous Perturbation Stochastic Approximation with compressive sensing. Applications to estimating gradient outer product matrix as well as standard optimization problems are illustrated via simulations.
\vspace{0.15in}

\noindent
{\it Keywords:} Gradient estimation; Compressive sensing; Sparsity; Gradient descent; Gradient outer product matrix.
\end{abstract}


\section{Introduction}\label{intro}
Estimating the gradient of a given function (with or without noise) is often an important part of  problems in reinforcement learning, optimization and manifold learning.  In reinforcement learning, policy-gradient methods are used to obtain an unbiased estimator for the gradient. The policy parameters are then updated with increments proportional to the estimated gradient \cite{PGM}. The objective is to learn a locally optimum policy. REINFORCE and PGPE methods (policy gradients with parameter-based exploration) are  popular instances of this approach (See \cite{NN} for details and comparisons, \cite{survey} for a survey on policy gradient methods in the context of actor-critic algorithms). In manifold learning, various finite difference methods have been explored for gradient estimation \cite{learninggrad}, \cite{learningrad1}. The idea is to use the estimated gradient to find the lower dimensional manifold where the given function actually lives. Optimization, i.e. finding maximum or minimum of a function, is a ubiquitous problem that appears in many fields wherein one seeks zeroes of the gradient. But the gradient itself might be hard to compute. Gradient estimation techniques prove particularly useful in such scenarios.

A further theoretical justification is facilitated by the results of \cite{TA}. In \cite{TA}, it was shown that given a connected and  locally connected metric probability space $(X, d, \mu)$ (i.e., $X$ is a compact metric space with metric $d$ and $\mu$ is a probability measure on the Borel $\sigma$-algebra of $(X, d)$), under suitable conditions, any function $f: X^n \mapsto \mathbb{R}$ is close (in $L^1(\mu^n)$) to a function $\int_X f(x_1, \cdots, y_i, \cdots, y_j, \cdots x_n)\cdots dy_i \ldots dy_j\cdots$ on a lower dimensional space. As a special case, a similar fact can be proved for real-values $1$-Lipschitz functions on $\mathbb{R}\setminus \mathbb{Z}$ with metric $| \cdot |_\infty^n$ (See Theorem 1.2 in \cite{TA}). This suggests that sparse gradients can be expected for functions on high dimensional spaces with adequate regularity conditions.

Over the years gradient estimation has also become an interesting problem in its own right. One would expect that the efficiency of a given method for gradient estimation also depends on the properties of function $f$. We consider one such class of problems in this paper. Suppose we have a continuously differentiable function $f: \mathbb{R}^n \mapsto \mathbb{R}$ where $n$ is large, such that the gradient $\nabla f$ lives mostly in a lower dimensional subspace. This means that one can throw out most of the coordinates of $\nabla f$ in a suitable local basis without incurring too much error. In this case, computing $\frac{\partial f}{\partial x_i} \; \forall i$ is clearly a waste of means. Dimensionality reduction techniques for optimization problems is an active area of research and several useful methods for this have been developed in the mathematical and engineering community \cite{HD1}, \cite{HD2}. If in addition the function evaluations are expensive, most gradient estimation methods become inefficient. Such is the case, e.g., if a single function evaluation is the output of a large time consuming simulation. This situation our specific focus. The problem of expensive function evaluations does not seem to have attracted much attention in machine learning literature. However, there has been quite a lot of work on optimization of functions with expensive evaluation. Most methods, however,  focus on learning a good surrogate for the original function (See \cite{BB3}, \cite{BB2}, \cite{BB1}).

To handle the first issue, ideas from compressive sensing can be applied. Compressive sensing theory tells us that an $s$-sparse vector can be reconstructed from $m \sim s \log (n/s)$ measurements. This means that one does not need the information about $\nabla f$ in all $n$ directions, a much smaller number of measurements would suffice. These ideas are frequently used in signal as well as image processing (see, e.g., \cite{image1, image2}). To remedy the latter difficulty, we use an idea from Simultaneous Perturbation Stochastic Approximation (SPSA) due to Spall \cite{spall}, viz., the Simultaneous Perturbation (SP).

We begin by explaining the proposed method for gradient estimation. Important ideas and results from compressive sensing and SPSA that are relevant to this work are discussed in section~\ref{CS} and section~\ref{GE} respectively. We state the main result in section~\ref{Main}. Section~\ref{App} consists of applications to manifold learning and optimization with simulated examples.

Some notational preliminaries are as follows. By $\| \cdot \|$ we denote the usual Euclidean norm in $\mathbb{R}^n$ as well as the Frobenius norm for matrices over $\mathbb{R}$. Throughout, `a.s.' stands for `almost surely', i.e., with probability one.


\section{Gradient Estimation: Combining Compressive Sensing and SP}

As mentioned above, if function evaluations are expensive, SP works well to avoid the problem of computing function multiple times. However, if the gradient is sparse it makes sense to use the ideas of compressive sensing to our advantage. Combining these two techniques helps us overcome the problem of too many function evaluations and also exploit the sparse structure of the gradient. The idea is to use SP to get sufficient number of  observations to be able to recover the gradient via $l_1$-minimization. We describe the method in detail in the following sub-sections.


\subsection{Compressive Sensing}\label{CS}

Assume that $\nabla f \in \mathbb{R}^n$ is an approximately sparse vector. The idea of compressive sensing is based on the fact that typically a sparse vector contains much less information or complexity than its apparent dimension. Therefore one should be able to reconstruct $\nabla f$ with considerable accuracy with much less information than that of order $n$. We will make these ideas more precise in the forthcoming discussion on compressive sensing. We state all the results for vectors in $\mathbb{R}^n$. All of these results also hold for vectors over $\mathbb{C}$. 
We start by defining what we mean by sparse vectors.
\begin{defi}[Sparsity]
The support of a vector $x \in \mathbb{R}^n$ is defined as:
$$ supp(x) := \{ j \in [n] : x_j \neq 0 \}. $$
where $[n] = \{ 1, 2, \ldots, n \}$. The vector $x \in \mathbb{R}^n$ is called $s$-sparse if at most $s$ of its entries are nonzero, i.e., if
$$ \|x\|_0 := card(supp(x)) \leq s. $$
\end{defi}

\

We assume that the observed data $y \in \mathbb{R}^m$ is related to the original vector $x \in \mathbb{R}^n$ via $Ax = y$ for some matrix $A \in \mathbb{R}^{m \times n}$, where $m < n$. In other words, we have a linear measurement process for observing $x$. The theory of compressive sensing tells us that if $x$ is sparse, then it can recovered from $y$ by solving a convex optimization problem. In particular, given a suitable matrix $A$ and appropriate $m$, the following $l_1$-minimization problem recovers $x$ exactly.
\begin{equation}
\min_{z \in \mathbb{R}^n} \| z \|_1 \; \; \mbox{subject to} \; \; y = Az
\end{equation}
where $y = Ax$ are the $m$ observations. These ideas were introduced by E. Candes and T. Tao in their seminal paper on near-optimal signal reconstruction \cite{CandesTao}. In this paper, the authors proved that the matrices suitable for the recovery need to have what is called the restricted isometry property (RIP). A large class of random matrices satisfy the RIP  with quantifiable `high probability' and are therefore suitable for reconstruction via $l_1$-minimization. In particular, subgaussian matrices have been shown to have RIP with high probability  and are suitable for the aforementioned reconstruction scheme for $m \sim s \log (n/s)$. This gives the explicit relationship between the sparsity level $s$, the dimension of the original vector $n$ and the dimension of the observed data $m$.
In recent times some work has been done to construct deterministic matrices with this restricted isometry property \cite{RIPdeter}. The current known lower bound on $m$ for deterministic matrices is of the order of $s^2$ where $s$ is the sparsity. Thus random matrices are a better choice for linear measurement for reconstruction via compressive sensing.

\

For the scope of this paper, we consider robust recovery options using Gaussian random matrices, i.e., matrices whose entries are realizations of independent standard normal random variables.

\begin{remark}
Matrices with more structure like random partial Fourier matrix or in general bounded orthonormal systems can also be used as meaurement matrices for compressive sensing techniques. Given a random draw of such a matrix with associated constant $K \geq 1$,  a fixed $s$-sparse vector $x$ can be reconstructed via $l_1$-minimization with high probability provided $m \geq CK^2 s \log n$. For more details on random sampling matrices in compressive sensing see chapter 12 of \cite{CSbook}.
\end{remark}

The crucial point here is that it is enough that the given vector is sparse in some basis. A more detailed discussion on various aspects of compressive sensing can be found in \cite{CSbook}. In real-life situations the measurements are almost always noisy. It may also happen that the original vector $x$ is not sparse but is close to a sparse vector. In other words, we would like the reconstruction scheme to be robust and stable. Theorem 9.13 of \cite{CSbook} gives explicit error bounds for stable and robust recovery where $A$ is a subgaussian matrix. The bound is expressed in terms of $\sigma_s(x) := \inf \{ \| x - z \|: z \in \mathbb{R}^n \ \mbox{is} \ s\mbox{-sparse} \}$, the distance of $x$ from the nearest $s$-sparse vector, and the measurement error. See \cite{StableCandesTao},\cite{error} and \cite{robustgaussian} for more on robust and stable recovery via compressive sensing. \\

We assume that our observations $y = (y_1, \ldots, y_m)$ are noisy. The following theorem gives an error bound on the reconstruction from noisy measurements using a Gaussian matrix.
\begin{theorem}[Theorem 9.20 in \cite{CSbook}] \label{robustgaussian}
Let $M \in \mathbb{R}^{m \times d}$ be a random draw of a Gaussian random matrix and $x \in \mathbb{R}^n$ be a $s$-sparse vector. Let $y = Mx + \xi$ be noisy measurements of $x$ such that $\| \xi \|_2 \leq \eta$. If for $0 < \epsilon <1$ and some $\tau > 0$,
\begin{equation} \label{mbound}
\frac{m^2}{m+1} \geq 2s \left( \sqrt{ \log(en/s)} + \sqrt{\frac{\log(\epsilon^{-1})}{s}} + \frac{\tau}{\sqrt{s}} \right)^2,
\end{equation}
then with probability at least $1 - \epsilon$ every $\hat{x}$ that minimizes  $\| z\|_1$ subject to $\| Mz - Mx \| \leq \eta$ approximates $x$ with $l_2$-error
$$ \|x - x^* \|_2 \leq \frac{2 \eta}{\tau}.$$
\end{theorem}

See Theorem 9.29 in \cite{CSbook} for a statement for stable and robust recovery via Gaussian matrices.

\subsection{Simultaneous Perturbation Stochastic Approximation} \label{GE}

As discussed above we have a fairly good reconstruction of a sparse gradient $\nabla f$ given a sufficient number of observations $\{y_i\}$. However, as mentioned before, the problem often is the unavailability of these observations. Even though observations for $\nabla f$ are not readily available, one may compute $y_i$'s using the available information, that is, noisy measurements of the function $f$. Note that we have, however, assumed that the function evaluations are computationally expensive. We will now address this issue of estimating $\nabla f$ with low computational overheads. \\

Let $e_i$ denote the $i^{th}$ coordinate direction for $1 \leq i \leq n$. We consider the finite difference approximation
$$ \frac{\partial f(x(k))}{\partial x_i} \approx \frac{f(x(k)+ \delta e_i) - f(x(k) - \delta e_i)}{2 \delta} $$
where $x(k) = (x_1(k), \ldots, x_n(k))$ and $\delta > 0$. By Taylor's theorem, the error of estimation is $O(\delta \| \nabla^2 f(x(k))\|)$ where $\nabla^2 f$ denotes the Hessian. This estimate requires $2n$ function evaluations. Replacing the `two sided differences' $(f(x(k)+ \delta e_i) - f(x(k) - \delta e_i))/2$ above by `one sided differences' $(f(x(k)+ \delta e_i) - f(x(k)))$  reduces this to $n + 1$, which is still large for large $n$. Given that we have assumed $f$ to be such that the function evaluations are computationally expensive, an alternative method is desirable. We use the method devised by Spall \cite{spall} in the context of stochastic gradient descent, known as Simultaneous Perturbation Stochastic Approximation (SPSA).  \\

Recall the stochastic gradient descent scheme \cite{book}
\begin{equation}
x(k+1) = x(k) + a(k)\left[-\nabla f(x(k)) + M(k+1)\right], \label{SAbasic}
\end{equation}
where:
\begin{itemize}
\item  $\{M(k)\}$ is a square-integrable martingale difference sequence, viz., a sequence of zero mean random variables with finite second moment satisfying
$$
E\left[M(k+1) | x(m), M(m), m \leq k\right] = 0 \ \forall \ k \geq 0,
$$
i.e., it is uncorrelated with the past. We assume that it also satisfies
\begin{equation}
\sup_kE\left[\|M(k+1)\|^2 | x(m), M(m), m \leq k\right] < \infty, \label{mgbdd}
 \end{equation}

  \item $\{a(k)\}$ are  step-sizes satisfying
  \begin{equation}
  a(k) > 0 \ \forall k, \ \sum_k a(k) = \infty, \ \sum_k a(k)^2 < \infty. \label{steps}
   \end{equation}
   \end{itemize}

The term in square bracket in (\ref{SAbasic}) stands for a  noisy measurement of the gradient. Under mild technical conditions, $x(k)$ can be shown to converge a.s.\ to a local minimum of $f$ \cite{book}. The idea is that the incremental adaptation due to the slowly decreasing step-size $a(k)$ averages out the noise $\{M(k)\}$, rendering this a close approximation of the classical gradient descent with vanishing error \cite{book}. In practice the noisy gradient is often unavailable and one has to use an approximation $\widehat{\nabla f}$ thereof using noisy evaluations of $f$, e.g., the aforementioned finite difference approximations, which lead to the Kiefer-Wolfowitz scheme. That is where the SP scheme comes in. We describe this next.\\

Let $\{\Delta_i(k), 1 \leq i \leq n, k \geq 0\}$ be i.i.d.\ zero mean random variables such that
\begin{itemize}
\item $\Delta(k) = (\Delta_1(k), \ldots \Delta_n(k))$ is independent of $M(\ell), \ell \leq k + 1$.
\item $P(\Delta_i(k) = 1) = P(\Delta_i(k) = -1) = 1/2$.
\end{itemize}

Then by Taylor's theorem, we have that for $\delta>0$:
\begin{equation} \label{spall}
\frac{f(x(k) + \delta \Delta(k)) - f(x(k))}{\delta \Delta_i(k)} \approx \frac{\partial f}{\partial x_i}(x(k)) + \sum_{j \neq i}\frac{\partial f}{\partial x_i}(x(k)) \frac{\Delta_j}{\Delta_i}.
\end{equation}
Note that since $\Delta_j$'s are i.i.d.\ zero mean random variables, we have for $j \neq i$,
$$\mathbb{E} \left[ \frac{\partial f}{\partial x_i}(x(k)) \frac{\Delta_j}{\Delta_i} \Big | x(m), M(m), m \leq k-1 \right] = 0.$$
Hence for the purpose of stochastic gradient descent, the second term in \eqref{spall} acts as a zero mean noise (i.e., martingale difference) term that can be clubbed with $M(k+1)$ as martingale difference noise and gets averaged out by the iteration. This serves our purpose, since the above scheme requires only two function evaluations per iterate given by
$$x_i(k+1) = x_i(k) + a(k)\left[ - \frac{f(x(k) + \delta \Delta(k)) - f(x(k))}{\delta \Delta_i(k)}\right] + M_i(k+1).$$
Our idea is to generate $\widetilde{\nabla f}$ according to the scheme discussed above.\\

It should be mentioned that Spall also introduced another approximation based on a single function evaluation (see, e.g., \cite{book}, Chapter 10). But this suffers from numerical issues due to the `small divisor' problem, so we do not pursue it here.\\



\subsection{Main result}\label{Main}
As mentioned in the introduction, the idea is to combine the SP and compressive sensing to obtain a sparse approximation of $\nabla f$. Note that while SP gives an estimate with zero-mean error, the final estimate of gradient obtained after compressive sensing may not be unbiased. To avoid the error from piling up we need to average out the error at SP stage. We propose the following algorithm for estimating gradient of $f$.   Let $a_i$ denote the row vectors of $A$.\\

\begin{algorithm}[H] \label{GEA}
\caption{Gradient Estimation at some $x \in \mathbb{R}^n$ with SP and Compressive Sensing}\label{basic}
{\bf Initialization:}
\\ $A = (a_{ij})_{m \times n} \gets$ random Gaussian matrix. \\
\begin{algorithmic}
\State $ \bullet \ y_i = \frac{f(x + \delta \sum\limits_{j}^m \Delta_j a_j) - f(x)}{\delta \Delta_i}$ \ \ for \ $i =1, \ldots, m$.

\

\State $\bullet$ Repeat and average over $y_i$'s $k$ times to get $\bar{y}_i$.

\

\State $\bullet \ y = (\bar{y}_1, \ldots, \bar{y}_m)= A{\nabla f}(x) + \eta$ where $\eta$ denotes the error.

\

\State $\bullet$ Solve the $l_1$-minimization problem to obtain $\widetilde{\nabla f}$: \\
$$ \mbox{minimize} \ \| z \|_1 \ \ \mbox{subject to} \ \ \|Az = y \| \leq \eta $$
\end{algorithmic}
\vspace{0.05in}
\bf Output: estimated gradient $\widetilde{\nabla f}(x)$.
\end{algorithm}

\noindent The following theorem states that with high probability such an approximation is \lq \lq close" to the actual gradient.

\begin{theorem} \label{mainthm}
Let $f: \mathbb{R}^n \mapsto \mathbb{R}$ be a continuously differentiable function with bounded sparse gradient. Then for $m \in \mathbb{N}$  such that it satisfies the bound in~\eqref{mbound}, $0 < \epsilon << \frac{1}{m}$, and given $\delta > 0$ (as in \eqref{spall}) and $\tau>0$ (as in Theorem~\ref{robustgaussian}), $\nabla f$ can be estimated by a sparse vector $\widetilde{\nabla f}$ such that with probability at least $1 - \epsilon m$,
$$\| \widetilde{\nabla f} - \nabla f \| < \frac{2t}{\tau}.$$
where, $t > 2m O(\delta)$.
\end{theorem}

\begin{proof}
Let $A \in \mathbb{R}^{m \times n}$ be a Gaussian matrix such that $m$ satisfies~\eqref{mbound}. Then, following the same idea as in~\eqref{spall}, we have:
\begin{eqnarray} \label{obs}
y_i &=& \frac{f(x + \delta \sum\limits_{j=1}^m \Delta_j a_j) - f(x)}{\delta \Delta_i} \nonumber \\
&=& \langle \nabla f(x), a_i \rangle + \sum_{j \neq i} \frac{\Delta_j \langle \nabla f, a_i \rangle}{\Delta_i} + O(\delta).
\end{eqnarray}
So we get
\begin{equation} \label{error}
y = A \nabla f + \ \mbox{`error'},
\end{equation}
where we quantify the `error' below.

\

The above computation is carried out $k$ times independently, keeping the matrix $A$ fixed and choosing the random vector $\Delta$ according to the distribution defined in section~\ref{GE}. The reason for this additional averaging is as follows. The reconstruction in compressive sensing need not give an unbiased estimate, since it performs a nonlinear (minimization) operation. Thus it is better to do some pre-processing of the SP estimate (which is nearly, i.e., modulo the $O(\delta)$ term, unbiased) to reduce its variance. We do so by repeating it $k$ times with independent perturbations and taking its arithmetic mean. This may seem to defeat our original objective of reducing function evaluations, but the $k$ required to get reasonable error bounds is not large as our analysis shows later, and the computational saving is still significant (see `Remark \ref{rem}' below). \\

Denote by $y^l$ the measurement obtained at $l^{th}$ iteration of SP. The error for a single iteration is given by
\begin{equation*}
\eta^l = \left(\sum_{j \neq 1} \frac{\Delta^l_j \langle \nabla f, a_1 \rangle}{\Delta^l_1} + O(\delta), \ldots, \sum_{j \neq m} \frac{\Delta^l_j \langle \nabla f, a_m \rangle}{\Delta^l_m} + O(\delta)\right).
\end{equation*}
Denote by $X_{ij}^l = \frac{\Delta^l_j \langle \nabla f, a_i \rangle}{\Delta^l_i}$. So, $X_{ij}^l$ are zero-mean conditionally (given past iterates) independent random variables.

\

The error vector after $k$ iterations is given by
\begin{eqnarray}
\eta & = & \frac{1}{k} \sum\limits_{l = 1}^k \eta^l \nonumber \\
& =&  \frac{1}{k} \sum\limits_{l = 1}^k \left( \sum_{j \neq 1} X_{1j}^l + O(\delta), \ldots, \sum_{j \neq m} X_{mj}^l + O(\delta) \right) \nonumber \\
& = & \left( \frac{1}{k} \sum\limits_{l = 1}^k \sum_{j \neq 1} X_{1j}^l + O(\delta), \ldots, \frac{1}{k} \sum\limits_{l = 1}^k \sum_{j \neq m} X_{mj}^l + O(\delta) \right). \label{etaerror}
\end{eqnarray}

 In order to apply the ideas from compressive sensing as in Theorem~\ref{robustgaussian}, we need to have a bound on the error $\| \eta \|$. This is obtained as follows. Let $K > 0$ be a constant such that the $O(\delta)$ term above is bounded in absolute value by $K\delta$. $K$ can, e.g., be a bound on $\|\nabla^2 f\|\|A\|$ by the mean value theorem, where we use the Frobenius norm.  Choose $C \geq \sup| \langle \nabla f, a_i \rangle |$ and $t > 2mK\delta$. Then, by Hoeffding's inequality we have,
\begin{eqnarray}
P(\|\eta \| \geq t)& \leq &  \sum\limits_{i=1}^m P\left( \Big| \frac{1}{k}\sum\limits_{l = 1}^k \sum_{j \neq i} X_{ij}^l + O(\delta) \Big| > t/m \right)  \nonumber \\
& \leq & \sum\limits_{i=1}^m P\left( \Big| \sum\limits_{l = 1}^k \sum_{j \neq i} X_{ij}^l \Big| > kt/2m \right)  \nonumber \\
& \leq & 2 m e^{-\frac{k t^2}{2 m^2(m-1)C^2}}.
\end{eqnarray}

Choose the number of iterations, $k > \frac{2 m^3 C^2}{t^2} \log \left( \frac{2}{\epsilon} \right)$. Then,
\begin{equation}\label{errorbound}
P(\|\eta \| \geq t) \leq  \epsilon m.
\end{equation}
We define  $\widetilde{\nabla f}$ to be the reconstruction of the gradient using $m$ measurements. That is, $\widetilde{\nabla f}$ solves the following optimization problem:
\begin{equation*}
\min_{z \in \mathbb{R}^n} \| z \|_1 \; \; \mbox{subject to} \; \; \; \; \| Az - y \| \leq \eta
\end{equation*}
where $y$ is as in \eqref{error}. Our claim then follows from bound in \eqref{errorbound} and Theorem~\ref{robustgaussian}.
\end{proof}

\begin{remark} \label{rem}
Note that the minimum number of iterations of SP required to obtain a \lq \lq good" estimate of $\nabla f$ is given by
\begin{eqnarray*}
k &>& \frac{2m^3C^2}{t^2}\log\left(\frac{2}{\epsilon}\right) \\
&\geq& \frac{mC^2}{2K^2\delta^4}\log\left(\frac{2}{\epsilon}\right) \\
&\geq& \frac{s\tilde{C}}{\delta^4}\log\left(\frac{n}{s}\right)\log\left(\frac{2}{\epsilon}\right).
\end{eqnarray*}
for a suitable constant $\tilde{C}$.
\end{remark}


The above $\widetilde{\nabla f}$ can now be used as an effective gradient in various problems involving gradients of high-dimensional functions. Two such applications are discussed in the next section.


\section{Applications} \label{App}
We consider the applications of our method to manifold learning and optimization problems. The gradient estimates obtained using our method can be used to estimate the gradient outer product matrix or can be plugged into an optimization scheme. In the former case, along with an example, we also provide error bounds on the estimated and actual gradient outer product matrix. For the latter case, we look at an example and provide suitable modifications to existing algorithms to achieve faster convergence. Algorithm 1 below which is based on Theorem~\ref{mainthm}, is used for gradient estimation.

There are various algorithms available for carrying out the $l_1$-minimization. A detailed discussion of these algorithms can be found in \cite{Algos}, Chapter 15 of \cite{CSbook}. Here we use the homotopy method.
\begin{algorithm}[H] \label{GEA}
\caption{Gradient Estimation at some $x \in \mathbb{R}^n$ with SP and Compressive Sensing}\label{basic}
{\bf Initialization:}
\\ $A = (a_{ij})_{m \times n} \gets$ random Gaussian matrix. \\
\begin{algorithmic}
\State ${y}(n) \ \  \gets A{\nabla f}(x(n))$ + error as obtained in equation~\eqref{error}.
\vspace{0.02in}
\State $\widetilde{\nabla f}(x) \gets l_1$ minimization using $Homotopy(y,A)$.
\end{algorithmic}
\vspace{0.05in}
\bf Output: estimated gradient $\widetilde{\nabla f}(x)$.
\end{algorithm}
\noindent Here $Homotopy(y, A)$ denotes the $l_1$-recovery from observations $y$ and Gaussian random matrix $A$ using the homotopy method.
All the simulations were performed on MATLAB using the available toolbox for $l_1$-minimization (Berkeley database: http://www.eecs.berkeley.edu/~yang/software/ l1benchmark/).

Consider a function $f: \mathbb{R}^{25000} \mapsto \mathbb{R}$ given by $f(x) = x^{T} MM^T x$ where, $M$ is $25000\times 3$-dimensional matrix with $3$ non-zero elements per row. Let $A$ be a random Gaussian matrix that is used for measurement. We consider $m = 50$ measurements.

\

Figure~\ref{fig:sim} shows the performance of the proposed method with varying number of SP iterations.

\begin{figure}[H]
\centering
\includegraphics[width = 0.7\textwidth]{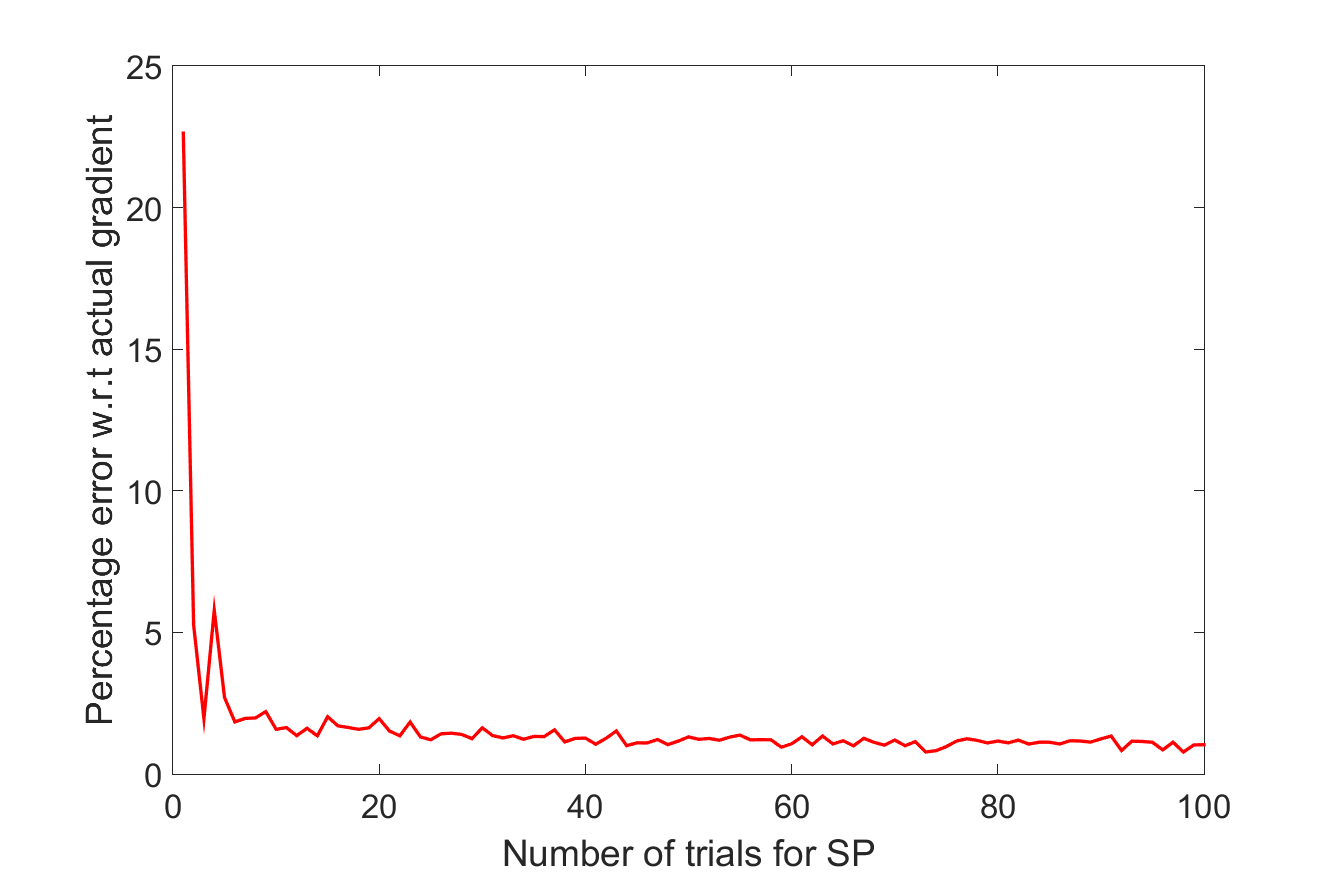}
\caption{Percentage error of $\| \nabla f - \widetilde{\nabla f} \|$ with varying number of iterations $k$ for SP.}\label{fig:sim}
\end{figure}

Figure~\ref{fig:comp1} and \ref{fig:comp1blown} show the comparison between our method and naive SP for estimating gradient with gradually increasing number of iterations for averaging over SP (The quantity `$k$' in (\ref{etaerror})). As mentioned earlier, since the gradient is assumed to be sparse, using naive SP to compute derivative in each direction seems wasteful. Although the error diminishes as the number of iterations for SP increase, the proposed method combining compressive sensing with SP consistently performs better.

\

Figures~\ref{fig:comp2} and \ref{fig:comp2blown} show that the proposed method works well with higher sparsity levels too. It also shows that with higher $s$, performance of naive SP improves. This is expected. The extremely high error in the naive SP method (especially for small $s$) is owing to the fact that the actual gradient is extremely sparse and in the beginning SP method ends up populating almost all the coordinates. That contributes to the high percentage of error as seen in the aforementioned figures.

\begin{figure}[!tbp]
  \centering
  \begin{minipage}[b]{0.4\textwidth}
    \includegraphics[width=\textwidth]{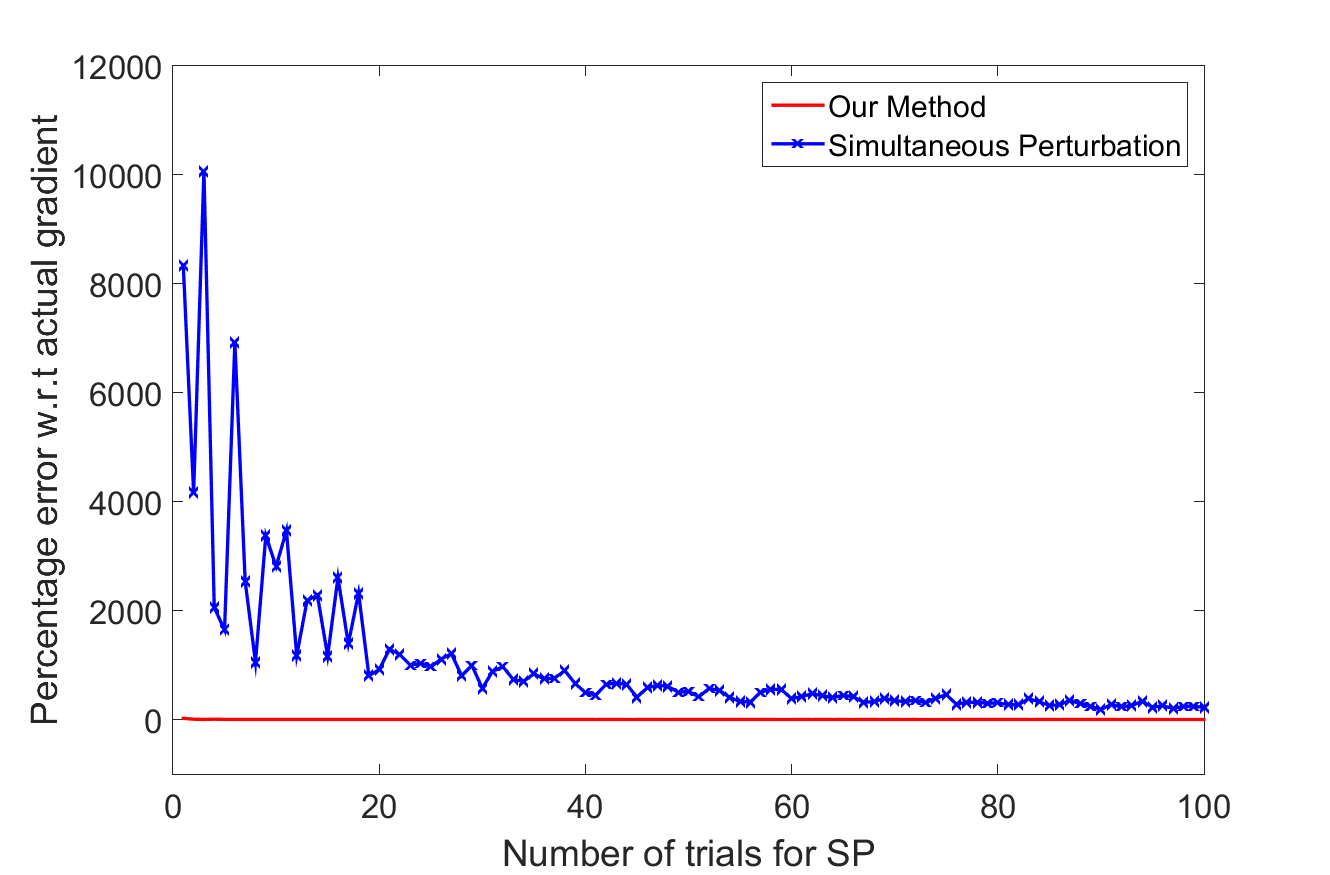}
    \caption{Performance of the proposed algorithm vs. the SP method with varying number iterations $k$ at SP step.} \label{fig:comp1}
  \end{minipage}
\hfill
  \begin{minipage}[b]{0.4\textwidth}
    \includegraphics[width=\textwidth]{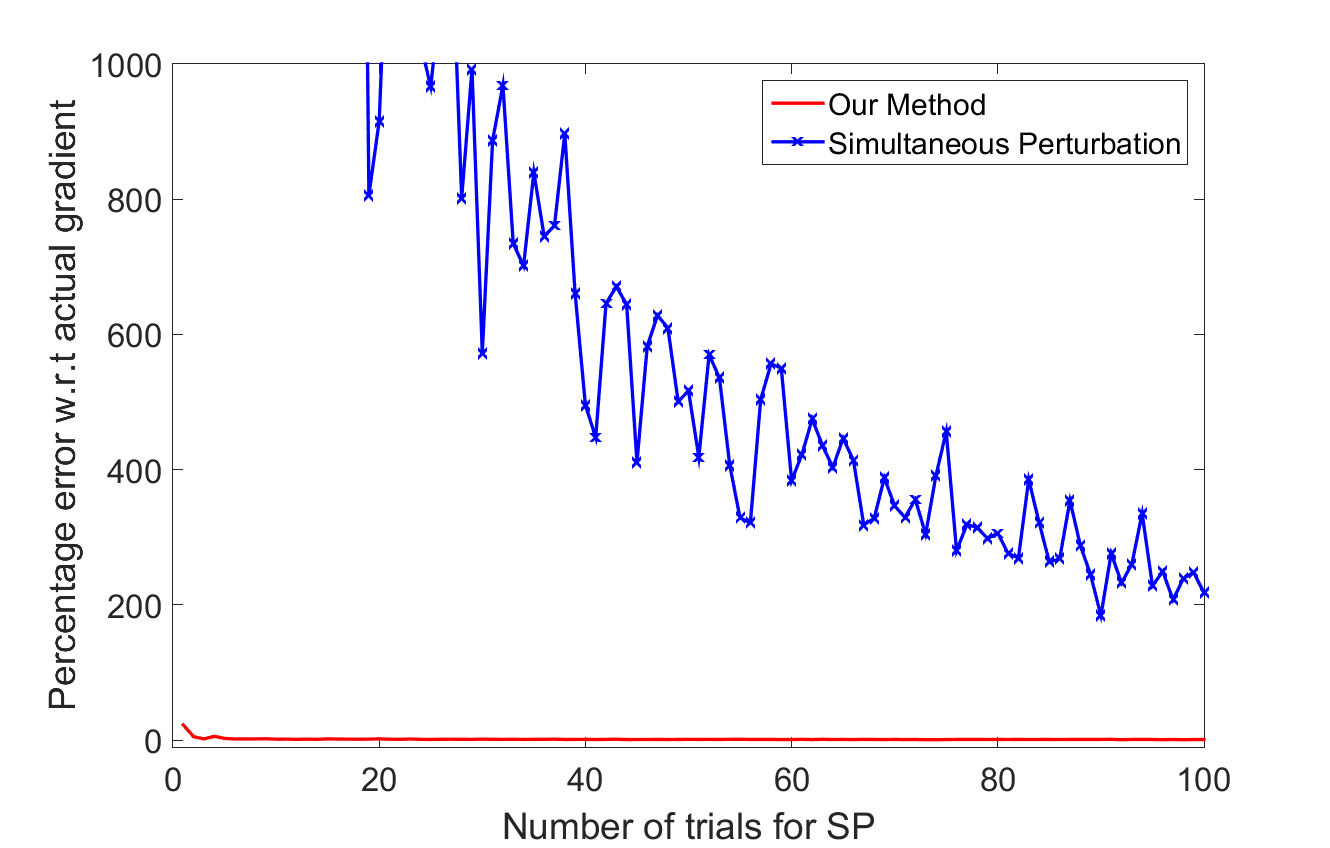}
    \caption{A closer look at Figure~\ref{fig:comp1} : Performance of the proposed algorithm vs. the SP method with varying number iterations $k$ at SP step.}\label{fig:comp1blown}
  \end{minipage}
\end{figure}


\begin{figure}[!tbp]
  \centering
  \begin{minipage}[b]{0.4\textwidth}
    \includegraphics[width=\textwidth]{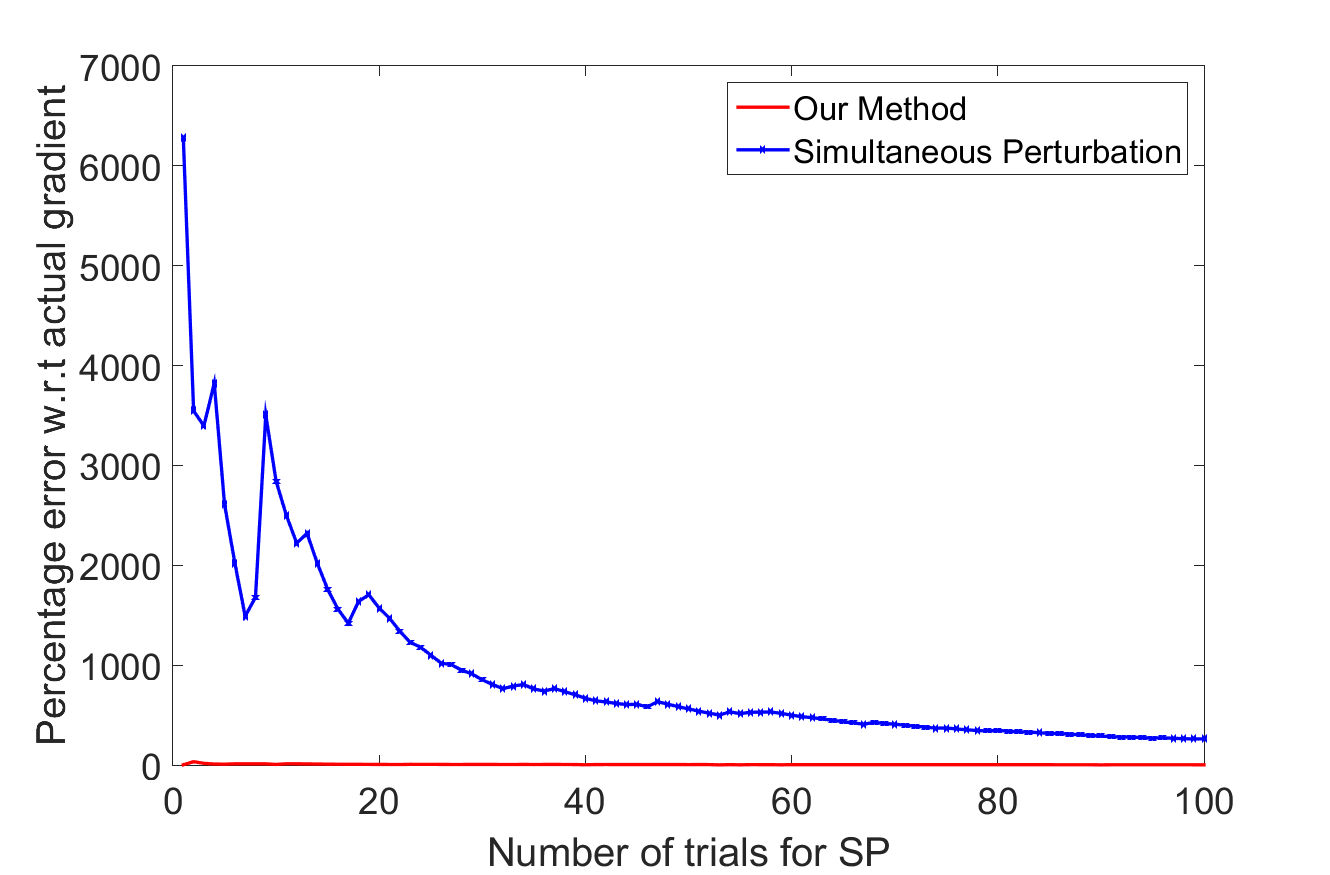}
    \caption{Performance of the proposed algorithm vs. the SP method with varying number iterations $k$ at SP step. Here $s=50$ and $m=500$.} \label{fig:comp2}
  \end{minipage}
\hfill
  \begin{minipage}[b]{0.4\textwidth}
    \includegraphics[width=\textwidth]{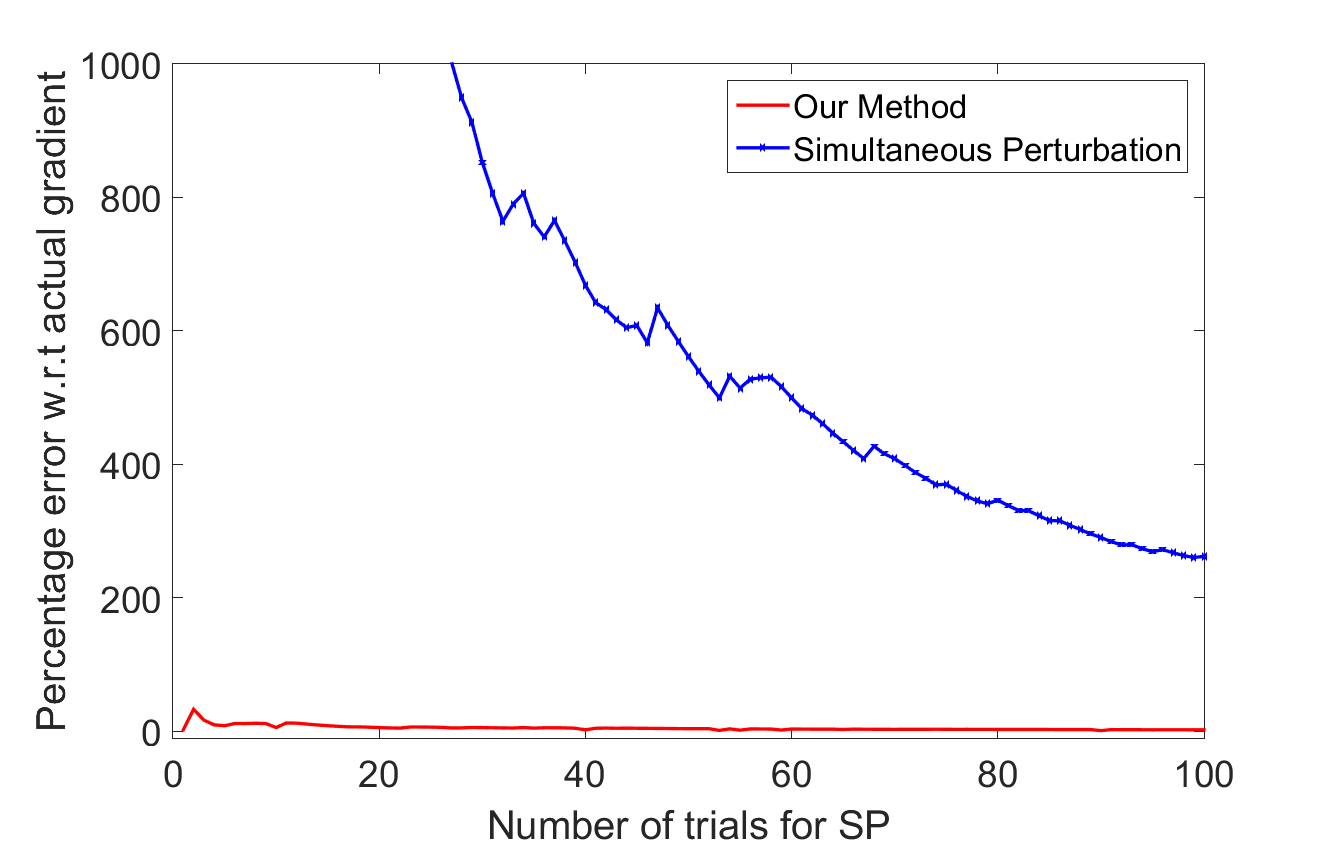}
    \caption{A closer look at Figure~\ref{fig:comp2}: Performance of the proposed algorithm vs. the SP method with varying number iterations $k$ at SP step. Here $s=50$ and $m=500$.} \label{fig:comp2blown}
  \end{minipage}
\end{figure}

Before we consider specific applications, we illustrate how the percentage error of estimated gradient with varying $k$ for different sparsity levels $s$. For appropriately large $m$, for small $k$ the error is high (this matches with the discussion in Remark~\ref{rem}). As $k$ increases the error is much less. As long as $m$ satisfies \ref{mbound}, the compressive sensing results apply. Figure~\ref{fig : s} shows the behaviour of the proposed method with variation in the sparsity, but with constant number of observations. We consider $10000$-dimensional vector with $m=50$ observations. As expected, for a fixed $m$, as the sparsity increases, increasing $k$ no longer helps as the compressive sensing results do not apply and the error increases. $f(x) =x(1)^2 + \ldots + x(s)^2$ was used as a test function for both the simulations.

\begin{figure}[H] \centering \includegraphics[width = 0.7\textwidth]{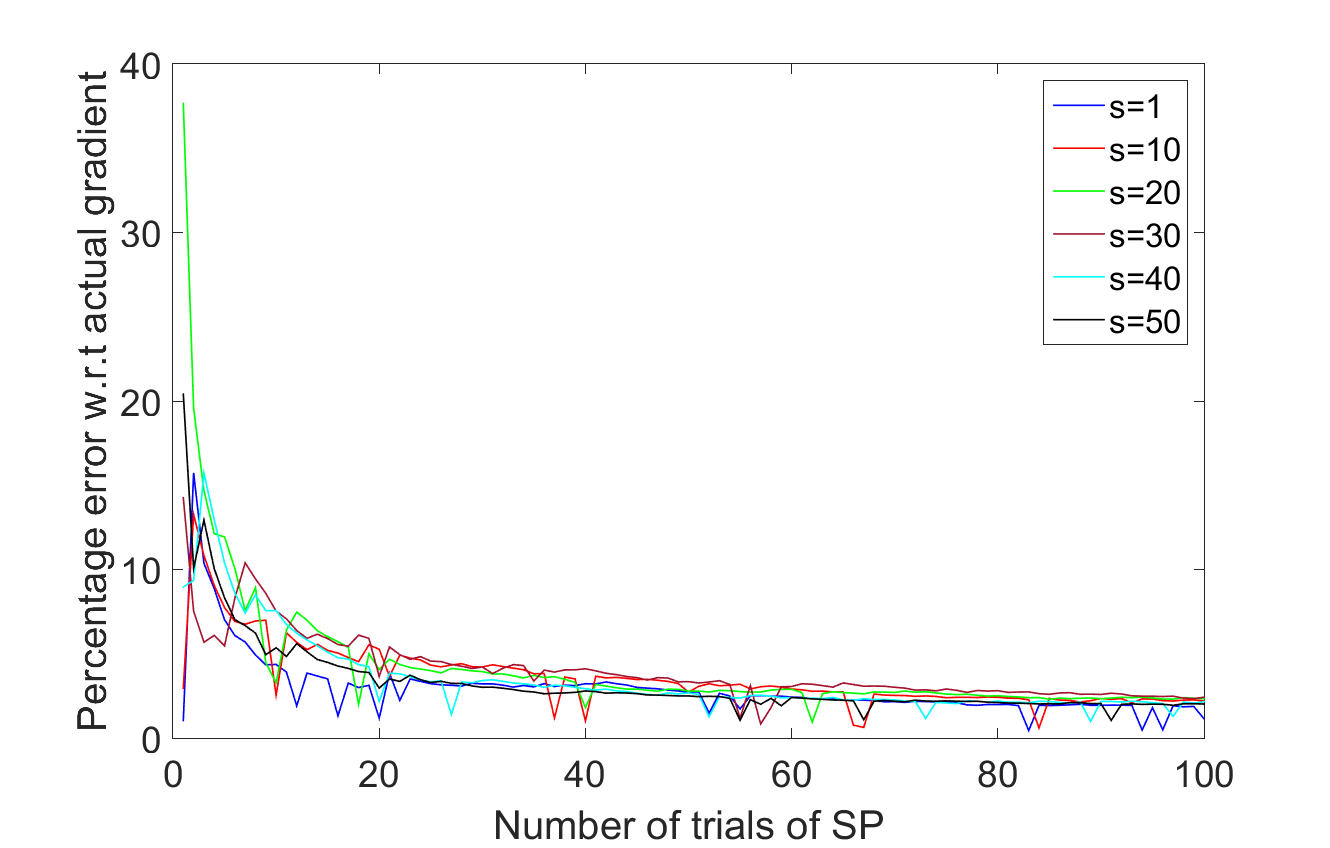} \caption{Performance of the proposed algorithm with variation $k$ for different sparsity levels.}\label{fig : s} \end{figure}

\begin{figure}[H] \centering \includegraphics[width = 0.7\textwidth]{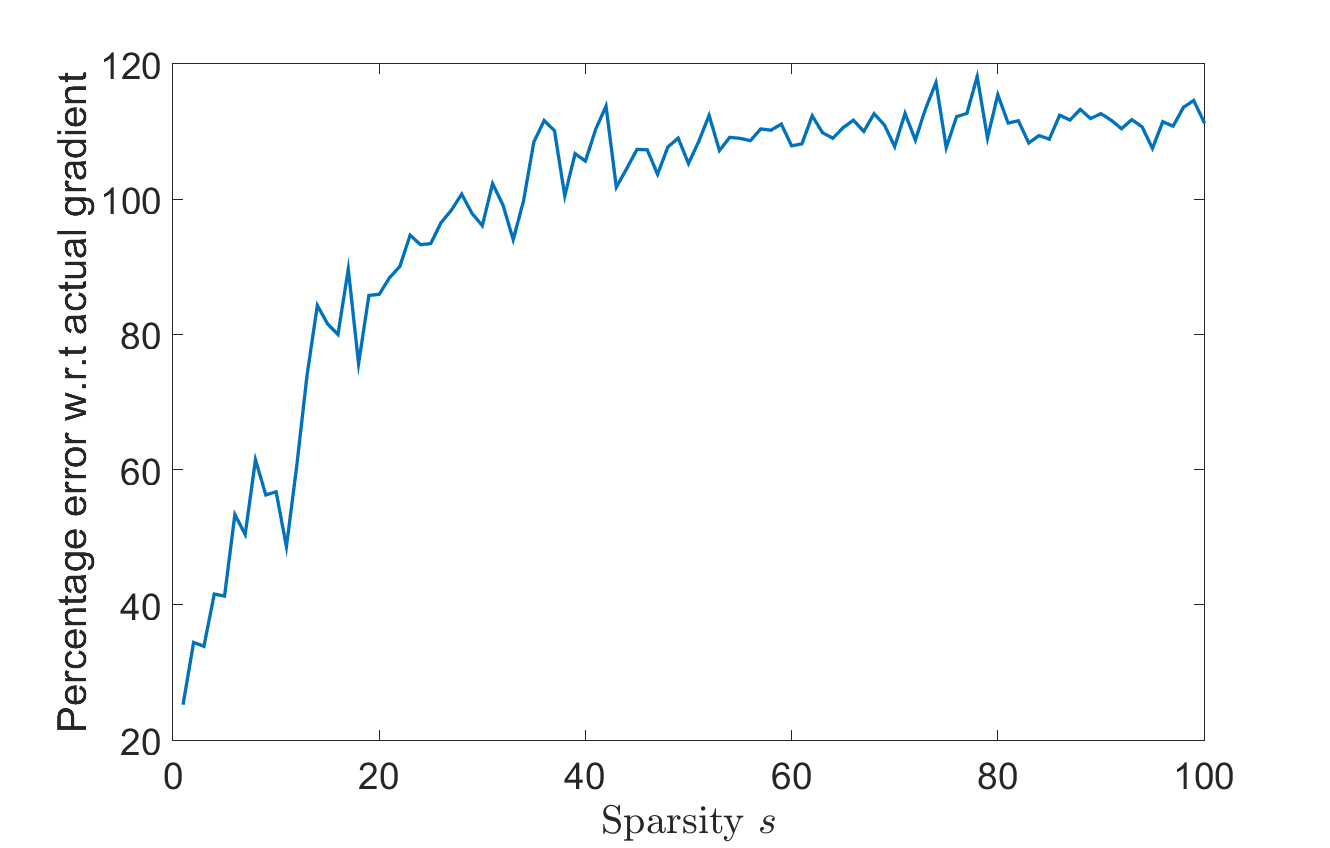} \caption{Performance of the proposed algorithm with variation in sparsity.}\label{fig : s} \end{figure}


\subsection{Manifold Learning: Estimating e.d.r.\ space}
Consider the following semi-parametric model
$$ Y = f(X) + \epsilon$$
where $\epsilon$ is noise and $f$ is a smooth function $\mathbb{R}^n \mapsto \mathbb{R}^m$ of the form $f(X) = g(b_1^TX, \ldots, b_d^T X)$. Define by $B$ the matrix $(b_1, b_2, \ldots, b_d)^T$. $B$ maps the data to a $d$-dimensional relevant subspace. This means that the function $f$ depends on a subspace of smaller dimension given by Range$(B)$ (Note that this is essentially the local view in manifold learning : $B$ can vary with location.). The vectors or the directions given by the vectors $b_i$ are called the effective dimension reducing directions or e.d.r. The question is: how to find the matrix $B$? It turns out that if $f$ doesn't vary  in some direction $v$, then $v \in Null(E_X[G])$ where $G$ is the gradient outer product matrix defined as
$$ G = [[ G_{ij} ]] \; \; \mbox{where} \; G_{ij} = \Big \langle \frac{\partial f}{\partial x_i}(X), \frac{\partial f}{\partial x_j}(X) \Big \rangle$$
and $E_X[ \ \cdot \ ]$ denotes the expectation over $X$. Lemma 1 from \cite{learningrad1} stated below implies that to find the e.d.r.\ directions it is enough to compute $E_X[G]$.
\begin{lemma}
Consider the semi-parametric model
\begin{equation} \label{edr}
Y = g(b_1^TX, \ldots, b_d^T X) + \epsilon,
\end{equation}
where $\epsilon$ represents zero mean finite variance noise. Then the espected gradient outer product (EGOP) matrix $G$ is of rank at most $d$. Furthermore, if $\{v_1, \ldots, v_d \}$ are the eigenvectors associated to the nonzero eigenvalues of $G$, following holds: $$ Span(B) = Span(v_1, \ldots, v_d). $$
\end{lemma}

\ \\

Clearly, calculating $E_X[G(X)]$ is computationally heavy. We therefore try to estimate this matrix. Several methods are known for estimating the EGOP and this has been a very popular problem in statistics for a while. The idea of using EGOP for obtaining e.d.r.\ originated in \cite{slicedregression}. While there are other methods based on inverse regression etc., most of the efforts have been directed towards getting an efficient way to estimate gradients in order to finally estimate EGOP (See \cite{adaptiveEDR}). In \cite{learninggrad}, the authors use their method of gradient estimation for this purpose. The idea is to use sample observations $\{f(x_i)\}$ for $\{x_i\}$ in a neighborhood of the given point $x$ and minimize over $z$ the error $$\frac{1}{n^2} \sum_{i, j = 1}^n w_{ij} [y_i - f(x_j) - \langle z, (x_i - x_j )\rangle]^2,$$ where  $w_{ij} \geq 0$ are weights (`kernel') that favor locality $x_i \approx x$ and are typically Gaussian, with regularization in a reproducing kernel Hilbert space (RKHS). The minimizer then is the desired estimate. In \cite{EPOG} a rather simple rough estimator using directional derivative along each coordinate direction is provided. The authors demonstrate that for the purpose of finding e.d.r., a rough estimate such as theirs suffices. We also propose a method via gradient estimation. Take $\widehat{G}$ to be the matrix defined by
$$\widehat{G}_{ij} = \Big \langle \widetilde{\frac{\partial f}{\partial x_i}}, \widetilde{\frac{\partial f}{\partial x_j}} \Big \rangle.$$
In other words, $\widehat{G} = \widetilde{\nabla f} \widetilde{\nabla f}^T$, where $\widetilde{\nabla f}$ denotes the estimate of $\nabla f$ obtained by algorithm 1. We impose our previous restrictions on $f$. That is, the function evaluations at any point are expensive and the gradient of $f$ is sparse. In this case we propose an estimate for $E_X[G]$ by the mean of $\widehat{G}$ over a sample of $r$ points given by the set $\chi = \{ (x_i, f(x_i)) \}_{1 \leq i \leq r}$. By $\langle \ \cdot \ \rangle$, we shall denote the empirical mean over the sample set $\chi$. Thus,

\begin{equation*}
\langle G(X)\rangle = \frac{1}{r} \sum\limits_{x_i \in \chi} \nabla f(x_i) \nabla f(x_i)^T
\end{equation*}
and
\begin{equation*}
\langle\widehat{G}(X)\rangle = \frac{1}{r} \sum\limits_{x_i \in \chi} \widetilde{\nabla f}(x_i) \widetilde{\nabla f}(x_i)^T.
\end{equation*}

\begin{theorem}
Let $f: \mathbb{R}^n \mapsto \mathbb{R}^k$ from the semi-parametric model in~\eqref{edr} be a continuously differentiable function with bounded sparse gradient. Then, for $0 < \epsilon << \frac{1}{m}$ and some $\tau > 0$, with probability at least $1 - \epsilon m $,
\begin{equation}
\left\| E_X[G] - \langle \widehat{G}\rangle \right\|_2 < \frac{6R^2}{\sqrt{r}} \left( \sqrt{\ln n} + \sqrt{\ln \frac{1}{\epsilon}} \right) + \frac{2t}{\tau} \left( \frac{2t}{\tau} + 2R \right)
\end{equation}
where $r$ is the sample size, $R$ is such that $\| \nabla f \|_2 \leq R$, $t$ is as in Theorem~\ref{mainthm} and $m \in \mathbb{N}$ is such that it satisfies the bound in~\eqref{mbound}.
\end{theorem}

The proof closely follows the line of argument in \cite{EPOG}.

\begin{proof} Note that,
$$\| E_X[G(X)] - \langle \widehat{G}(X)\rangle \|_2 \leq \| E_X[G(X)] - \langle G(X)\rangle \|_2 + \| \langle G(X)\rangle - \langle \widehat{G}(X)\rangle \|_2.$$
The idea is to bound each term. We use concentration inequality for sum of random matrices (See Theorem 2.1, lecture 23 \cite{MC} and \cite{RM} for more general results) to claim that for $\epsilon >0$,
$$ \| E_X[G(X)] - \langle G(X)\rangle \|_2 \leq \frac{6R^2}{\sqrt{r}} \left( \sqrt{\ln n} + \sqrt{\ln \frac{1}{\epsilon}} \right)$$
with probability $1-\epsilon$. For the second term, it is enough to show that it is bounded for any single sample point $x$. Observe that for any two vectors $v$ and $w$,
$\| vv^T - ww^T \|_2 \leq \| (v - w)^T (v + w) \|_2 $
Using this we get, for a fixed $x$,
\begin{eqnarray}
\| G(x) - \widehat{G}(x) \|_2 & =& \| \nabla f(x) \nabla f(x)^T - \widetilde{\nabla f}(x) \widetilde{\nabla f}(x)^T \|_2 \nonumber \\
& \leq & \| \nabla f(x) - \widetilde{\nabla f}(x) \|_2 \| \nabla f(x) + \widetilde{\nabla f}(x) \|_2 \nonumber  \\
& \leq & \| \nabla f(x) - \widetilde{\nabla f}(x) \|_2 \left( \| \nabla f(x) - \widetilde{\nabla f}(x) \|_2 + 2 \| \nabla f(x) \|_2 \right) \nonumber \\
& \leq & \frac{2 t}{\tau} \left( \frac{2 t}{\tau} + 2R \right)
\end{eqnarray}
with probability $1 - \epsilon m$, where the last inequality is obtained by applying the bound from Theorem~\ref{mainthm}.
\end{proof}

We now simulate an example to illustrate the decay of the error $\| \langle G(X)\rangle - \langle \widehat{G}(X)\rangle \|_F$ (See Figure~\ref{fig:G}). Consider a function $f:R^{25000} \mapsto R^{3}$ given by:
$f_{i}(x)  = x^{T} M_{i}M^{T}_{i}x $ where, $M_i$ is a $25000$-dimensional vector with $3$ non-zero elements and $f_{i}(x)$ corresponds to the $i^{th}$ dimension of $f(x)$.
A $25000 \times 100$ Gaussian random matrix is used for the compressive sensing part of the algorithm. The plot of percentage in normed error between  $\langle G(X)\rangle$ and $\langle \widehat{G}(X)\rangle$, i.e. $||\langle G(X)\rangle - \langle \widehat{G}(X)\rangle||^{2}_F$ is shown below by varying number of samples $r = 1$ to $25$. Remember that due to the bias at compressive sensing step, we need to average out the gradient estimation error at SP step. This is done in $k = 100$ iterations.

\bigskip

\begin{figure}[H]
\centering
\includegraphics[width = 0.7\textwidth]{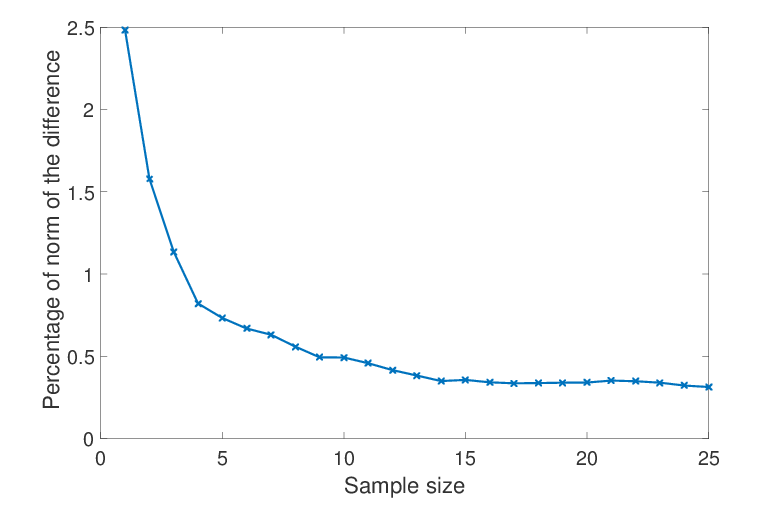}
\caption{Percentage error in $\| \langle G(X)\rangle - \langle \widehat{G}(X)\rangle \|$ with number of samples $r$ varying from $1$ to $25$.} \label{fig:G}
\end{figure}

Learning e.d.r. by estimating the gradient using the method proposed in this paper was compared with the SGL (Sparse Gradient Learning) method proposed in \cite{learninggrad} using the same function as above and an exponential kernel (See \url{http://www2.stat.duke.edu/~sayan/soft.html} for details). Here, $n=10000$ and the measurement matrix is a $10000 \times 100$ Gaussian matrix. The SP step is averaged over $100$ iterations. $100$ samples were considered for the SGL method with the neighborhood radius of of $0.05$. Sparsity of the gradient vector is $30$.


\begin{figure}[H]
\centering
\includegraphics[width = 0.7\textwidth]{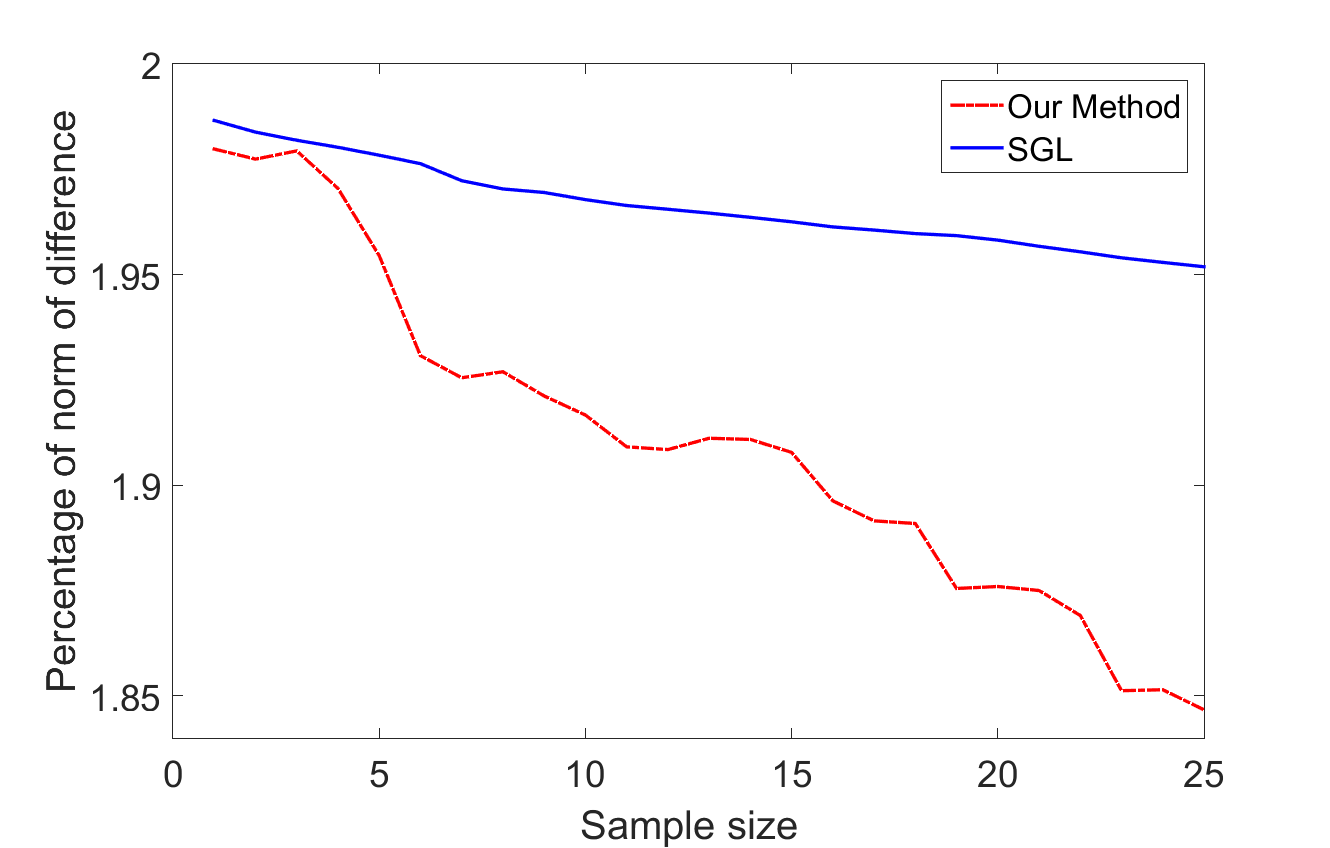}
\caption{Comparison of percentage error of $\| \langle G(X)\rangle - \langle \widehat{G}(X)\rangle \|$ computed by plugging gradient estimates by proposed method vs SGL method.}\label{fig:Gcomp}
\end{figure}



\subsection{Optimization}

We consider next a typical problem of function minimization, but only consider a function with sparse gradient. In other words, we want to minimize $f(x)$ where
$$f : \mathbb{R}^n \mapsto \mathbb{R}$$
is a continuously differentiable real-valued Lipschitz function such that function evaluation at a point in $\mathbb{R}^n$ is typically expensive. We also assume that $n$ is large and that $\nabla f$ is sparse. In addition, we assume that the critical points of $f$ (i.e., the zeros of $\nabla f$) are isolated. (This is generically true unless there is overparametrization.)
The idea is to use the  stochastic gradient scheme (\ref{SAbasic}) with the standard assumptions (\ref{mgbdd}), (\ref{steps}).
 It follows from the theory of stochastic approximation (See \cite{book}, Chapter 2) that under above conditions, the solution of the random difference equation (\ref{SAbasic}) tracks with probability one the trajectory of the solution of a limiting o.d.e.\ as long as the iterates remain bounded, which they do under mild additional conditions on $f$.  Following \cite{book}, Chapter 2, we use this so called `o.d.e.\ approach' which states that the algorithm will a.s.\ converge to the equilibria of the limiting o.d.e., which is
\begin{equation} \label{ODE}
\dot{x}(t) = -\nabla f(x(t)).
\end{equation}
For this, $f$ itself serves as the Lyapunov function, leading to the conclusion that the trajectories of (\ref{ODE}) and therefore a.s., the iterates of (\ref{SAbasic}) will converge to one of its equilibria, viz., the critical points of $f$. In fact under additional conditions on the noise, it will converge to a (possibly random) stable equilibrium thereof (\textit{ibid.}, Chapter 4).


\


 The stochastic gradient scheme requires $\nabla f(x)$ at each iteration. The problem, as noted,  often is the unavailability of $\nabla f(x)$. It is therefore important to have a good method for estimating the gradient. Typically one would obtain noisy measurements and hence the estimate will have a non-zero error $\eta$. It is known that if the error remains small, the iterates converge a.s.\ to a small neighbourhood of some point in the set of equilibria of  \eqref{ODE}. We analyze the resultant error below.  Also,  the error obtained in SP is zero-mean modulo higher order terms, so one can even take an empirical average over a few separate estimates in order to reduce variance. For high dimensional problems, the number of function evaluations remains still small as compared with, e.g., the classical Kiefer-Wolfowitz scheme. We use Theorem~\ref{mainthm} to justify using SP (Simultaneous Perturbation Stochastic Approximation) combined with compressed sensing to obtain an approximation for the gradient and then use the above scheme to minimize $f$.

\

\noindent Consider the following stochastic approximation scheme:
\begin{equation} \label{SAgrad}
x_{n+1} = x_n  + a(n)[-\nabla f(x_n) + M_{n+1} + \eta(n)]
\end{equation}
where $\{ \eta(n) \}$ is the additional error arising due to the error in gradient estimation. That is, $\widetilde{\nabla f}(x_n) = \nabla f(x_n) - \eta(n)$. If $\sup_n \| \eta(n) \| < \epsilon_0$ for some small $\epsilon_0$, then the iterates of \eqref{SAgrad} converge to a small neighbourhood $A$ of some point $x^*$ in $H = \{x: \nabla f(x) = 0 \}$ (See \cite{SA3} and chapter 10 of \cite{book}). This is ensured by a Lyapunov argument as follows.  The limiting o.d.e.\ is of the form
$$\dot{x}(t) = -\nabla f(x(t)) + \breve{\eta}(t)$$
for some measurable $\breve{\eta}(\cdot)$ with $\|\breve{\eta}(t)\| \leq \epsilon_0 \ \forall t$. Then
 $$\frac{d}{dt}f(x(t)) = -\|\nabla f(x(t))\|^2 + \langle\nabla f(x(t)), \breve{\eta}(t)\rangle,$$
 which is $< 0$ as long as $\|\nabla f(x(t))\|^2 > |\langle\nabla f(x(t)), \breve{\eta}(t)\rangle |$. Therefore $x(t)$ will converge to the set
 $$\{x : \|\nabla f(x)\| \leq \epsilon_0\}.$$
  Assume that the Hessian $\nabla^2f(x^*)$ is positive definite, which is generically so for isolated local minima. Then for $A$ small enough, the lowest eigenvalue $\lambda_m(x)$ of $\nabla^2f(x)$ for $x \in A$ is $> 0$. By mean value theorem, $\nabla f(x) = \nabla^2 f(x')(x - x^*)$ for some $x' \in A$, so $\|\nabla f(x)\| \geq \lambda_m(x')\|x - x^*\|$. Thus there is convergence to a ball of radius $\frac{\epsilon_0}{\lambda_m}$ around $x^*$. (A statement to this result without the estimate on the radius of the ball is contained in Theorem 1 of \cite{SA1}.) Thus we have:\\

\begin{theorem}
The stochastic gradient scheme
$$x_{n+1} = x_n  + a(n)[-\widetilde{\nabla f}(x_n) + M_{n+1}]$$
a.s.\ converges to a ball of radius $O(\epsilon_0)$ centered at some local minimum of $f$,
where $\widetilde{\nabla f}$ is the reconstructed gradient as in Theorem~\ref{mainthm} and $\epsilon_0$ is a bound on $\|\widetilde{\nabla f} - \nabla f\|$.
\end{theorem}

\noindent \textit{Proof}  The claim is immediate from the above observations about the perturbed differential equation and  Theorem 6, pp.\ 58-59, \cite{book}. \hfill $\Box$

\

Observe that we have only discussed asymptotic convergence above. For  real-life optimization problems, however, we must ensure that the scheme in \eqref{SAgrad} converges to a neighbourhood of $x^*$ in finite time. This is indeed true and recent concentration-type results (See \cite{SA2}, \cite{SA4}) strengthen the theoretical basis for plugging $\widetilde{\nabla f}(x)$ in place of $\nabla f(x)$ in stochastic gradient descent schemes. The results in \cite{SA2} involve estimates on lock-in probability, i.e., the probability of convergence to a stable equilibrium given that the iterates visit its domain of attraction. An estimate on the number of steps needed to be within a prescribed neighborhood of the desired limit set with a prescribed probability is also obtained. Specifically, the result states that if the $n_0$th iterate is in the domain of attraction of a stable equilibrium $x^*$, then after a certain number of additional steps, the iterates remain in a small tube around the differential equation trajectory converging to $x^*$ with probability exceeding
$$1 - \mbox{O}\left(e^{-\frac{C}{(\sum_{m = n_0}^{\infty}a(m)^2)^{\frac{1}{4}}}}\right),$$
\textit{ipso facto} implying an analogous claim for the probability of remaining in a small neighborhood of $x^*$ after a certain number of iterates. We refer the reader to \cite{SA2} for details. In \cite{SA4}, an improvement on this estimate is proved under additional regularity conditions on $\nabla f$ (twice continuous differentiability) using Alekseev's formula. We have omitted the details of both the cases as it needs much additional notation to replicate them here. These would, however, apply to the exact stochastic gradient descent. Since we have an additional error due to approximate gradient as in the preceding theorem, we need to combine the results of \textit{ibid.}  with the above theorem to make a weaker claim regarding how small the neighborhood of $x^*$ in question can be. Furthermore, these claims are about iterates which are in the domain of attraction of a stable equilibrium. This, however, is not a problem, as `avoidance of traps' results as in section 4.3 of \cite{book} (see also \cite{Benaim}, \cite{Bran}, \cite{Pem}) ensure that if the noise is rich enough in a certain precise sense, unstable equilibria are avoided with probability one.

\begin{remark}
Note that the gradient descent  is a stochastic approximation scheme which itself averages out the noise. So in principle  the averaging over $k$ steps at the SP stage in the original algorithm can be skipped. This means that for a stochastic gradient descent scheme, we cut down the cost of function evaluation even further. The simulations in the next section confirm that good results are obtained without averaging over SP iterations. There is, however, a standard trade-off involved between per step computation / speed of convergence, and fluctuations (equivalently, variance) of the estimates: any additional averaging improves the latter at the expense of the former.
\end{remark}

\

\subsection{ Numerical experiments}

We compare following three algorithms.

\

\begin{enumerate}

\item \textit{\large Actual Gradient Descent}

This is the classical stochastic gradient descent with exact gradient.

\

\begin{algorithm}\label{GDA}
\caption{Stochastic Gradient Decent with Compressive Sensing}\label{basic}
\vspace{0.05in}
{\bf Initialization:}
\vspace{0.05in}
\\ $x(0) = x_{initial},  A \gets random\ Gaussian\ matrix $\\
$ a(n)$ be a sequence that satisfies the properties of stepsize listed above. \\

{\bf Iteration:} \it Repeat until convergence criteria is met at $ n = n^\#$. At \ $n^{th}$ iteration:
\vspace{0.05in}
\begin{algorithmic}
\State ${y}(n) \ \  \ \ \  \  \gets A{\nabla f}(x(n))$ + \mbox{error}
\State $\widetilde{\nabla f}(x(n)) \gets l_1-\mbox{ minimization with} \ Homotopy(y(n),A)$
\State $x(n+1) \ \gets x(n) - a(n)[ \widetilde{\nabla f}(x(n))]$
\end{algorithmic}
\vspace{0.05in}
\bf Output: $Approximate \ minimizer\ of\ f\ i.e.\ x(n^\#)$
\end{algorithm}

\

\

\item \textit{\large Accelerated Gradient Method}

Accelerated gradient scheme was proposed by Nesterov \cite{AGS}. While Gradient Descent algorithm has a rate of convergence of order $1/s$ after $s$ steps, Nesterov's method achieves a rate of order $1/s^2$. We implement the method here to achieve an improvement in the time complexity further. The idea is to replace the $n^{th}$ iteration above by the following.

\

\begin{algorithm} \label{AGDA}
\vspace{0.05in}
At $ n^{th}$ iteration:
\vspace{0.05in}
\begin{algorithmic}
\State ${y}(n) \ \  \ \ \  \  \gets A{\nabla f}(x(n))$ + \mbox{error}
\State $\widetilde{\nabla f}(x(n)) \gets l_1-\mbox{ minimization with} \ Homotopy(y(n),A)$
\State $ z(n+1) \ \gets \ x(n) - a(n) [\widetilde{\nabla f}(x(n)) $
\State $ x(n+1) \ \gets \  (1 - \gamma(n)) z(n+1) + \gamma(n) z(n) $
\end{algorithmic}
\vspace{0.05in}
where, $\lambda$ and $\gamma$ are as follows:
\begin{equation*}
\lambda(0) = 0, \ \lambda(n) = \frac{1 + \sqrt{1+ 4 \lambda^{2}(n-1)}}{2}, \ \text{and} \ \gamma(n) = \frac{1 - \lambda(n)}{\lambda(n+1)}.
\end{equation*}
\end{algorithm}

This gives us faster convergence towards the minimum.

\item \textit{\large Adaptive Method}

Another way to achieve a faster convergence rate is to perform the $l_1$-minimization adaptively with the gradient descent. The idea is to again use the homotopy method for $l_1$-minimization but this part of the algorithm is run for very few iterations. The intermediate approximation of $\nabla f$ is then used for performing the stochastic gradient descent. As expected, the errors are high in the beginning but the convergence is faster.

\

\noindent We consider the following function to test our algorithms:
\begin{equation} \label{fn} f(x) = (x^T M_{1}^T M_{1} x)^3 +(x^T M_{2}^T M_{2} x)^2+x^T M_{1}^T M_{2} x \end{equation}
where, function $f: \mathbb{R}^{n} \mapsto \mathbb{R}$ and $M_{1},M_{2}$ are $n \times 3$ random matrices. This is to ensure sparsity of the gradient. Here, $n=25000$ and number of non-zero entries in each column of $M_1$ and $M_2$ are $s=3$. Number of measurements, $m=50$. $A$ is a $25000 \times 50$ random Gaussian matrix.\\

Figure~\ref{fig:SGL} and \ref{fig:SGLsmalln} show the comparisons between various algorithms described above for the same function.

\end{enumerate}
\begin{figure}[H]
\centering
\includegraphics[width = 0.6\textwidth]{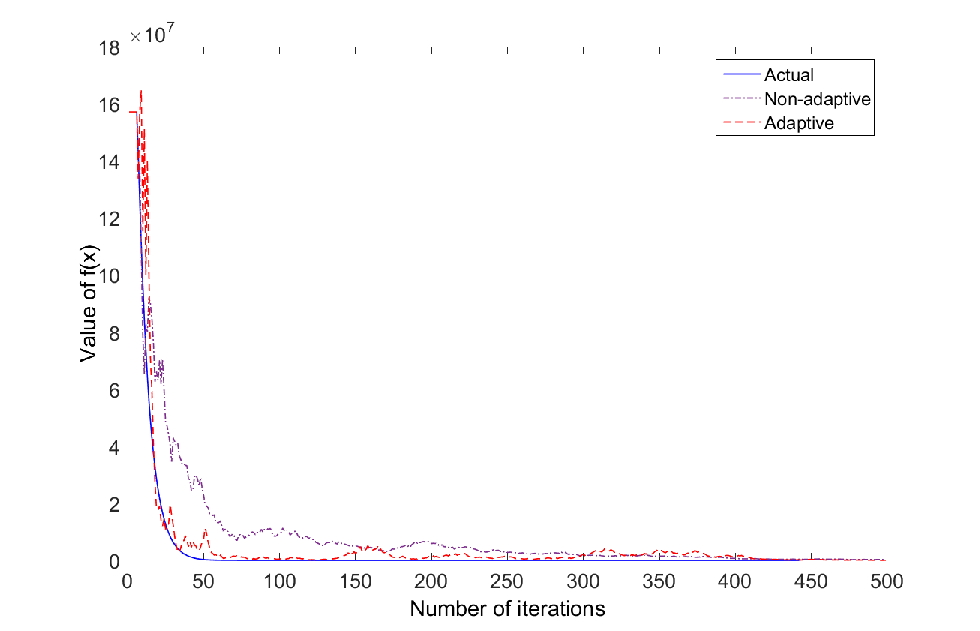}
\caption{Comparison of Gradient descent with actual gradient $\nabla f$ and estimated gradient $\widetilde{\nabla f}$
using non-adaptive and adaptive schemes.}\label{fig:SGL}
\end{figure}

\begin{figure}[H]
\centering
\includegraphics[width = 0.6\textwidth]{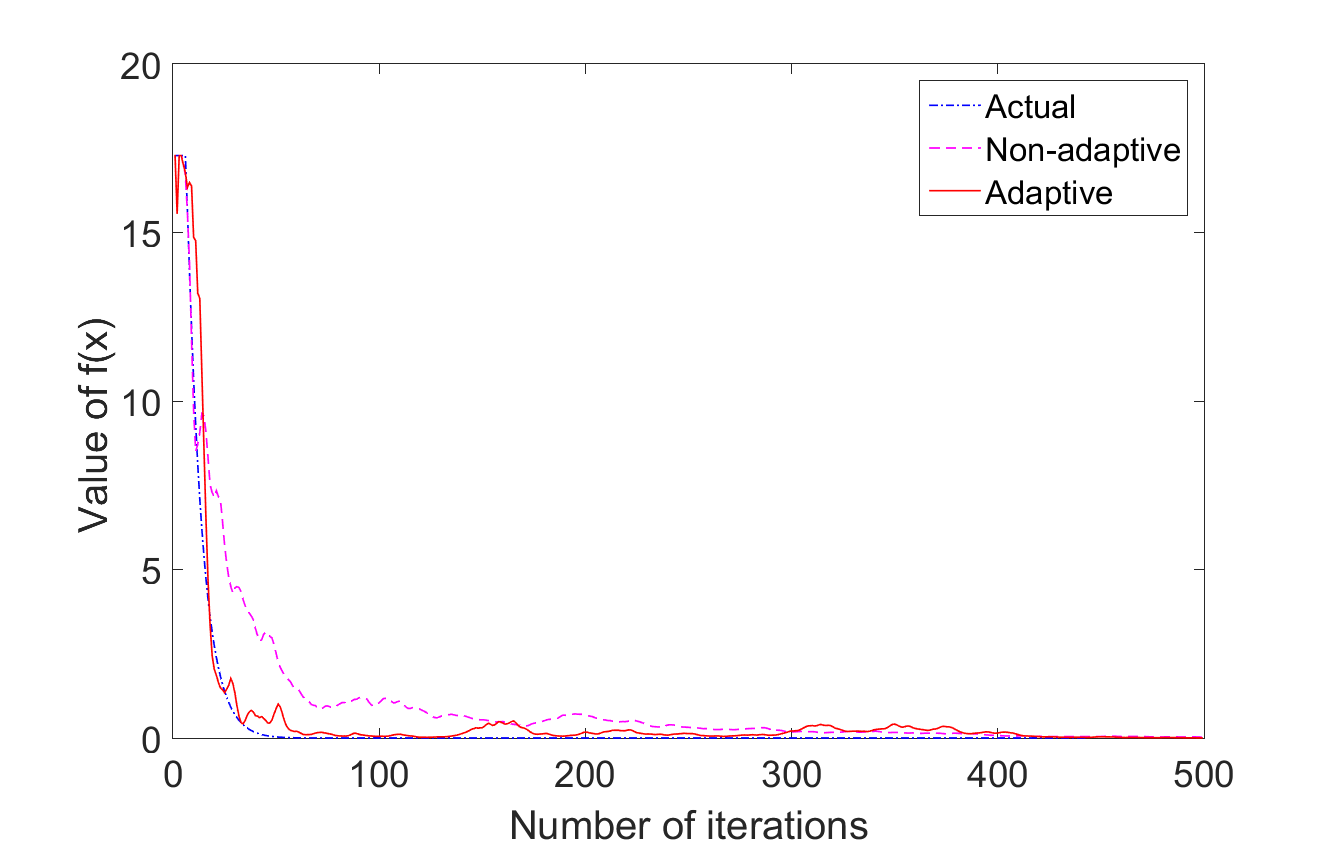}
\caption{A scaled down version of Figure~\ref{fig:SGL} for n=10000.} \label{fig:SGLsmalln}
\end{figure}

\bigskip

\begin{center}
\begin{tabular}{|c|c|c|c|}
\hline
\text{Algorithm Used} & \text{Time taken in Sec for n = 1000} &\text{n = 10000} & \text{ n = 25000} \\ \hline
\text{Adaptive Method} & 5.910 &30.4& 39\\
\text{With out Adaptive} & 50.572 &213.6& 489\\
\text{Actual Gradient Method} & 171.150 &5873.8& ~9000\\
\hline
\end{tabular}
\end{center}

As expected, adaptive method turns out to be faster compared to the non-adaptive method which in turn is much faster than the algorithm that computes actual gradients. Incidentally, the classical scheme all but converges in under 400 iterations. Even so it takes more time than the other two which take more iterations. This is because of the heavy per iterate computation  for the classical scheme. From the above table it is clear that as the dimensionality of the problem increases, adaptive method proves more and more useful compared to the other two algorithms.

We compared our method with the method proposed in \cite{learninggrad}. The function in \eqref{fn} is used for the comparison. Here, $n=10000, s=50$. Number of samples of SGL were $10$, chosen with neighbourhood radius $0.05$.

\begin{figure}[H]
\centering
\includegraphics[width = 0.6\textwidth]{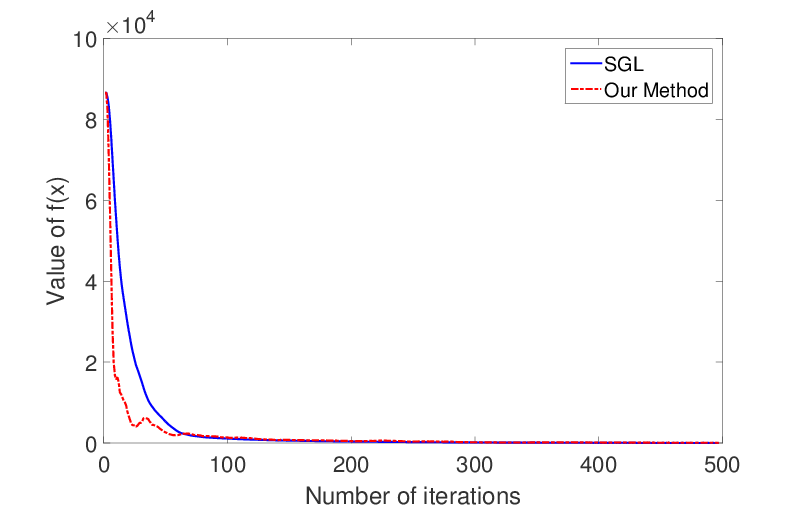}
\caption{Comparison of Gradient descent using estimated gradient from proposed method vs the SGL method proposed in \cite{learninggrad}.}
\end{figure}

In the following scaled down version $n=10000, s=100$ and $m=500$.
\begin{figure}[H]
\centering
\includegraphics[width = 0.6\textwidth]{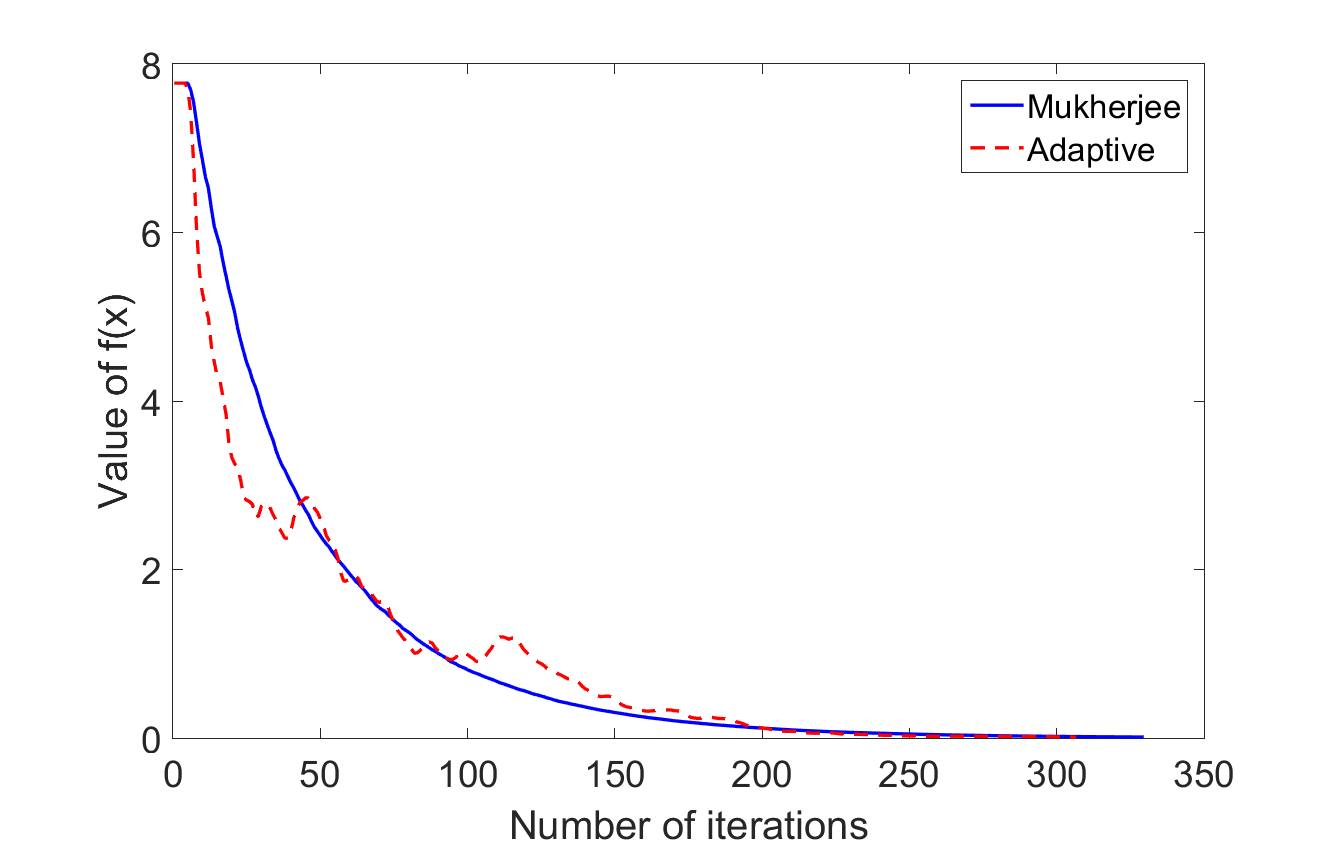}
\caption{Comparison of Gradient descent using estimated gradient from proposed method vs the SGL method proposed in \cite{learninggrad}.}
\end{figure}

The time taken by the SGL method and our method was $993$ and $83$ seconds respectively. As mentioned earlier, the aim of this paper is to provide a good estimation of gradient when the function evaluations are expensive. In such cases, our method would provide a significant gain in terms of  function evaluations needed. While in this example we do see a significant improvement in time taken for the estimation, there is no a priori reason to always expect it. It will indeed be the case when the function evaluations are `expensive' in terms of the time they take. One expects this to be so when the ambient dimension is high.


\section{Concluding remarks}
We have proposed an estimation scheme for gradient  in high dimensions that combines ideas from Spall's SPSA with compressive sensing and thereby tries to economize on the number of function evaluations. This has theoretical justification by the results of \cite{TA}. Our method can be extremely useful when the function evaluation is very expensive, e.g., when a single evaluation is the output of a long simulation. This situation does not seem to have been addressed much in literature. In very high dimensional problems with sparse gradient, computing estimates for partial derivatives in every direction is inefficient because of the large number of function evaluations needed. SP simplifies the problem of repeated function evaluation by concentrating on a single \textit{random} direction at each step. When the gradient vectors in such cases live in a lower dimensional subspace, it also makes sense to exploit ideas from compressive sensing. We have computed the error bound in this case and have also shown theoretically that this kind of estimation of gradient works well with high probability for the gradient descent problems and in other high dimensional problems such as estimating EGOP in manifold learning where gradients are actually low-dimensional and gradient estimation is relevant. Simulations show that our method works much better than pure SP.

\bigskip


\section*{Acknowledgment}
We thank GPU Centre of Excellence, IIT Bombay for providing us with the facility to carry out simulations and Prof.\ Chandra Murthy of Indian Institute of Science for helpful discussions regarding compressive sensing.

\bigskip


\end{document}